%% file: main.tex
\documentclass[10pt]{article} 
\usepackage[preprint]{tmlr}

\input{math_commands.tex}

\usepackage{caption}

\usepackage{latexsym}
\usepackage{amsmath}
\usepackage{amsthm,amssymb}
\usepackage{booktabs}
\usepackage{enumitem}
\usepackage{graphicx}
\usepackage{color}
\usepackage{subfigure}
\usepackage{hyperref}
\usepackage{url}

\usepackage{algorithm}
\usepackage{algorithmic}

\newtheorem{problem}{Problem}
\newtheorem{observation}{{Observation}}

\newtheorem{theorem}{Theorem}
\newtheorem{lemma}[theorem]{Lemma}

\newtheorem{definition}{Definition}

\theoremstyle{remark}
\newtheorem{remark}[theorem]{Remark}

\usepackage{soul}

\usepackage{adjustbox}
\usepackage{footnote}
\usepackage{nicefrac}
\usepackage[usestackEOL]{stackengine}
\usepackage{float}
\makesavenoteenv{tabular}
\usepackage{pgfplots}
\usepgfplotslibrary{colorbrewer}

\usepackage[colorinlistoftodos]{todonotes}
\usepackage{tikz} 
\usepackage{multirow}
\usetikzlibrary{arrows}
\tikzset{
    vertex/.style={circle,draw,minimum size=1.5em},
    edge/.style={->,> = latex'}
}
\usetikzlibrary{arrows.meta}
\usepackage{enumitem}
\usepackage{lipsum}

\newcommand{\supp}{{\rm supp}}

\title{Fairness in Monotone $k$-submodular Maximization: Algorithms and Applications}

\author{\name Yanhui Zhu \email yanhui@iastate.edu \\
      \addr Department of Computer Science\\
      Iowa State University
      \ANDA
      \name Samik Basu \email sbasu@iastate.edu \\
      \addr Department of Computer Science\\
      Iowa State University
      \ANDA
      \name A. Pavan \email pavan@cs.iastate.edu\\
      \addr Department of Computer Science\\
      Iowa State University}

\begin{document}

\maketitle

\begin{abstract}
Submodular optimization has become increasingly prominent in machine learning and fairness has drawn much attention. In this paper, we propose to study the fair $k$-submodular maximization problem and develop a $\nicefrac{1}{3}$-approximation greedy algorithm with a running time of $\mathcal{O}(knB)$. To the best of our knowledge, our work is the first to incorporate fairness in the context of $k$-submodular maximization, and our theoretical guarantee matches the best-known $k$-submodular maximization results without fairness constraints. In addition, we have developed a faster threshold-based algorithm that achieves a $(\nicefrac{1}{3} - \epsilon)$ approximation with $\mathcal{O}(\frac{kn}{\epsilon} \log \frac{B}{\epsilon})$ evaluations of the function $f$. Furthermore, for both algorithms, we provide approximation guarantees when the $k$-submodular function is not accessible but only can be approximately accessed. We have extensively validated our theoretical findings through empirical research and examined the practical implications of fairness. Specifically, we have addressed the question: ``What is the price of fairness?" through case studies on influence maximization with $k$ topics and sensor placement with $k$ types. The experimental results show that the fairness constraints do not significantly undermine the quality of solutions.
\end{abstract}

\input{sections/1-introduction}

\input{sections/2-preliminaries}

\input{sections/3-algorithms}

\input{sections/4-experiments}

\input{sections/5-conclusion}

\bibliography{references}

\bibliographystyle{tmlr}

\end{document}

%% file: math_commands.tex
\usepackage{amsmath,amsfonts,bm}

\def\eqref#1{equation~\ref{#1}}

\def\1{\bm{1}}

\DeclareMathAlphabet{\mathsfit}{\encodingdefault}{\sfdefault}{m}{sl}
\SetMathAlphabet{\mathsfit}{bold}{\encodingdefault}{\sfdefault}{bx}{n}

\DeclareMathOperator*{\argmax}{arg\,max}

%% file: sections/1-introduction.tex
\section{Introduction}
\label{sec:intro}
Submodular function optimization is a fundamental optimization problem, applied extensively across diverse fields such as feature compression, deep learning, sensor placement, and information diffusion, among others~\cite{bateni2019categorical,el2022data,kempe2003maximizing,padmanabhan2018influence,li2023submodularity}. Over the last decade, various versions of submodular optimization problems have garnered substantial attention from the research community.

While submodular functions can model interesting complex interactions in many applications, some problems cannot be modeled adequately by submodular functions. For example,  consider the problem of influence maximization with $k$ topics~\cite{ohsaka2015monotone}. Influence maximization involves identifying a seed set within a social network that can achieve the greatest possible spread of information. This subset selection problem is frequently modeled as a submodular maximization problem~\cite{kempe2003maximizing}. However, if the information spread includes multiple topics with varying effects for each topic on the network, the problem becomes more complex. Specifically, a seed set contains elements with each being assigned a specific topic. In such cases, the standard submodular maximization approach is insufficient to find a good seed set. 

Another example is the sensor placement with $k$ types. Consider the case where we want to monitor multiple statistics through sensors, e.g., temperature, humidity, and lights. Only one sensor can be placed at a location. The goal is to place sensors in the most informative locations to collect as much information as possible~\cite{Krause2007NearoptimalOS}. The standard submodular maximization model fails to account for scenarios where each sensor can be good at detecting a specific measure. 

To effectively handle the above applications where each topic or type of element has varying effects on the overall utility, one can formulate the problems as $k$-submodular functions optimization problem~\cite{kolmogorov2011submodularity,huber2012towards}. A formal definition of $k$-submodular functions is provided in the Preliminaries section.

\textbf{Hardness.} Maximizing a monotone $k$-submodular function even without constraints is challenging. Iwata et al.~\cite{Iwata2015ImprovedAA} achieved an approximation guarantee of $\frac{k}{2k-1}$ using $\mathcal{O}(kn)$ number of function evaluations.
The authors also showed that achieving an approximation ratio better than $\frac{k+1}{2k}$ even for the unconstrained case is NP-hard.  
Ohsaka et al.~\cite{ohsaka2015monotone} initialized the study of size-constrained maximization problems, including individual size (IS) and total size (TS). IS constraint specifies that the number of elements selected for each type $i \in [k]$ should not exceed an upper bound $u_i$, while the TS constraint specifies that the solution size should be at most $B$ regardless of the types. 
The approximation algorithms due to Ohsaka et al.~\cite{ohsaka2015monotone} for IS and TS size-constrained problems admit $1/3$ and $1/2$-approximation ratios, respectively.

\textbf{Fairness in $k$-submodular maximization.}
The solutions produced by the TS-constrained algorithms can be both over-represented and underrepresented. 
Although the IS-constrained problem upper-bounds the number of elements selected for each type, the solutions can still be underrepresented. Thus, the solutions may not guarantee {\em fairness}; fairness refers to the property
of solution where the constituents of solution of each type are neither under-represented nor over-represented. 
In short, the constraint over the solution states that the number of elements of each type must fall within specified
lower and upper bounds (each type may have different bounds). This notion of fairness using range constraints
has been examined in various previous works~\cite{celis2018fair,chierichetti2017fair,el2020fairness}, where 
the range constraints are used to ensure a balanced representation of different types/attributes of entities in the overall solution (for example, race, ethnicity, gender).



For example, in an influence maximization problem with $k$ groups, we are given a set $V$ of $n$ users, where each user can be assigned a type/attribute $i$. The task is to select a subset of users $\mathbf{X}=(X_1,\cdots,X_k)$ where $X_i$'s 
are $k$ disjoint groups, and in each group, the people share the same type/attribute. The solution $\mathbf{X}$
is fair if $\ell_i \leq |X_i| \leq u_i$ for a given choice of lower and upper bounds $\ell_i,u_i$ for each group $i\in [k]$ and $\sum_{i=1}^k |X_i| \leq B$ for overall size bound $B$.


In this context of fairness, a natural question is: {\em can one develop fast approximation algorithms with theoretical guarantees that address the fair $k$-submodular maximization problem?} In this paper, we address this problem.
The problem gets even more challenging when the submodular function $f$, under consideration, is unavailable; instead, it is only possible to query a surrogate function $\Tilde{f}$ that \textit{approximates} $f$. This is referred
to as approximate oracle for $f$ ~\cite{cohen2014sketch}. 
We show that our solution strategy can handle this scenario with approximate oracles as well.



\subsection{Our Contributions}
Firstly, we formulate the {\em fair $k$-submodular maximization} problem and then propose a $\nicefrac{1}{3}$-approximation greedy algorithm {\sc Fair-Greedy} that requires $\mathcal{O}(knB)$ oracle calls. Our work is among the first to consider fairness in $k$-submodular maximization and present approximation guarantees \footnote{We note that a parallel work \cite{hu2024approximation} proposed an alternative fairness formulation for this problem when this work was under review.}. Note that in submodular maximization, it is standard to measure the runtime using the oracle calls or oracle evaluations -- number of times the underlying ($k$-)submodular function $f$ is evaluated.

We incorporate the threshold technique and propose a faster algorithm {\sc Fair-Threshold} that achieves $(\nicefrac{1}{3}-\epsilon)$-approximation and runs within $\mathcal{O}(\frac{kn}{\epsilon}\log \frac{B}{\epsilon})$ evaluations of function $f$. The runtime of {\sc Fair-Threshold} is improved from linearly dependent on $B$ to logarithmically dependent on $B$, making the algorithm significantly faster for large $B$.

For both {\sc Fair-Greedy} and {\sc Fair-Threshold}, we derive approximation guarantees when approximate
oracle $\Tilde{f}$ is available instead of $f$.

Lastly, we empirically validate our theoretical results by designing experiments and comparing them with various baselines. We empirically address the question: \textit{``What is the price of fairness?"} with the applications of influence maximization with $k$ topics and sensor placement with $k$ types.

\subsection{Related Work}
Recent research has explored constrained bi-submodular and $k$-submodular maximization, with applications in sensor placement and feature selection~\cite{singh2012bisubmodular,ward2014maximizing}. This section provides an overview of known results on $k$-submodular maximization.  For size-constrained problems. Ohsaka et al.~\cite{ohsaka2015monotone} introduced greedy algorithms that achieve $1/2$ and $1/3$ approximation guarantees for TS and IS constraints, respectively. Stochastic methods~\cite{mirzasoleiman2015lazier} and threshold techniques~\cite{badanidiyuru2014fast} have also been employed to reduce query complexity~\cite{ohsaka2015monotone,nie2023size}. 
In another approach, Qian et al.~\cite{qian2017constrained} proposed a multi-objective evolutionary algorithm for TS-constrained problems, demonstrating a $1/2$-approximation with $\mathcal{O}(knB \log ^2 B)$ oracle evaluations. Matsuoka et al.~\cite{matsuoka2021maximization} leveraged curvature properties in $k$-submodular functions, weak $k$-submodularity, and approximate $k$-submodularity to show improvements in the approximation ratios. Additionally, Zheng et al.~\cite{zheng2021maximizing} provided theoretical guarantees for scenarios where the objective function is only approximately $k$-submodular.

In the context of knapsack constraints, Tang et al.~\cite{Tang2021OnMaximizing} developed an algorithm inspired by~\cite{khuller1999budgeted,Sviridenko2004ANO} that achieves a $\frac{1}{2}(1-\frac{1}{e})$ approximation with $\mathcal{O}(k^4n^5)$ function evaluations. Chen et al.~\cite{Chen2022Monotone} proposed a partial-enumeration algorithm, inspired by~\cite{khuller1999budgeted}, achieving a $\frac{1}{4}(1-\frac{1}{e})$ approximation with $\mathcal{O}(kn^2)$ evaluations. For the streaming setting, Pham et al.~\cite{pham2021streaming} introduced an algorithm that guarantees a $\frac{1}{4}-\epsilon$ approximation with $\mathcal{O}(\frac{kn}{\epsilon}\log B)$ oracle evaluations. 

In the case of matroid constraints, Sakaue et al.~\cite{sakaue2017maximizing} proposed an algorithm with a $1/2$-approximation, while Matsuoka et al.~\cite{matsuoka2021maximization} developed an algorithm that achieves a $\frac{1}{1+c}$ approximation, where $c$ represents the curvature. Ene et al.~\cite{ene2022streaming} advanced the research in streaming scenarios by proposing an algorithm that attains up to a 0.3-approximation using $\mathcal{O}(nk)$ evaluations.

Though there has been a considerable body of work in $k$-submodular maximization, none of these works study the notion of fairness.

\textbf{Organization.}
Section~\ref{sec:prelim} introduces the preliminary definitions, notations, and formal problem statement for the fair $k$-submodular maximization. In Sections~\ref{sec:fair-greedy-monotone} and \ref{sec:fast-fair-greedy}, we present the greedy and threshold-based faster algorithms, along with the theoretical analysis. Section~\ref{sec:experiements} covers the applications, experimental setup, datasets, baseline algorithms, and results.

%% file: sections/2-preliminaries.tex
\section{Preliminaries}
\label{sec:prelim}

Let $V$ be the ground set of size $n$. A $k$-submodular function $f$ is defined on $(k+1)^V = \{ (X_1, \cdots, X_k) \mid X_i \subseteq V ~\text{for~all~} i \in [k] \text{~and~} X_i \cap X_j = \emptyset ~~ \forall i \neq j \}$. 

For $\mathbf{A} = (X_1, \cdots, X_k)$, $\mathbf{B} = (Y_1, \cdots, Y_k) \in (k+1)^V$, we denote $\mathbf{A} \preceq \mathbf{B}$ if $X_i \subseteq Y_i, ~ \forall i \in [k]$. The marginal gain of adding an element $e$ with type $i$ to $\mathbf{A}$ is $\Delta_{e, i} f(\mathbf{A}) \triangleq f((X_1, \cdots, X_i \cup \{e\}, \cdots, X_k)) - f((X_1, \cdots, X_k))$. 
Denote the indicator vector format of a single element-type pair $(e,i)$ as $\mathbb{I}_{(e,i)}$.
Thus, for brevity, we have  $\mathbf{A} + \mathbb{I}_{(e,i)} = (X_1, \cdots, X_i \cup \{e\}, \cdots, X_k)$. 
In this paper, we use vector representation and set representation interchangeably; in particular, we use $\supp(\cdot)$ to convert a vector to a set, specifically, $\supp(\mathbf{A}) = \{e \in V \mid \mathbf{A}(e) \neq 0 \}$ and $\supp_i(\mathbf{A}) = \{e \in V \mid \mathbf{A}(e) =i \}$. 

The $k$-submodular functions can be formally defined as follows.
\begin{definition}[$k$-submodular functions]
    For $\mathbf{A}, \mathbf{B} \in (k+1)^V$,
    a function $f:(k+1)^V \rightarrow \mathbb{R}$ is $k$-submodular if for any $\mathbf{A} \preceq \mathbf{B}$, $e \in V \setminus \supp(\mathbf{B}), i\in [k]$, 
    \[\Delta_{e, i} f(\mathbf{A}) \geq \Delta_{e, i} f(\mathbf{B}).\] 


    $f$ is monotone if for $\mathbf{A} \preceq \mathbf{B}$, $f(\mathbf{A}) \leq f(\mathbf{B})$.
\end{definition}

In some applications, such as influence diffusion, only an approximate oracle value is efficiently accessible. Thus, we present the definition of approximate $k$-submodular functions. 

\begin{definition}[$\delta$-approximate $k$-submodular \cite{zheng2021maximizing}]
\label{def:approximate}
    For a real number $\delta \in [0,1)$, a function $\Tilde{f}:(k+1)^V \rightarrow \mathbb{R}$ is $\delta$-approximate $k$-submodular to a $k$-submodular function $f$ if for any $\mathbf{A} \in (k+1)^V$,
    \[ (1-\delta)f(\mathbf{A}) \leq \Tilde{f}(\mathbf{A}) \leq (1+\delta)f(\mathbf{A}). \]
\end{definition}
If $\delta=0$, $\Tilde{f}$ is the exact oracle.

Given a total budget $B$ across all types, an upper limit $u_i$ and lower limit $\ell_i$ for the number of elements selected with type $i \in [k]$, we define the family of {\em fair} solutions as follows. 
\begin{definition}
\label{def:feasible}
The family of feasible sets for fair $k$-submodular is
\begin{align*}
    \mathcal{F} = \{ &\mathbf{s} \in (k+1)^V \mid \ell_i \leq |\supp_i(\mathbf{s})| \leq u_i\\
    &~\text{for~all}~i\in [k] ~\text{and}~\sum_{i=1}^k|\supp_i(\mathbf{s})| \leq B \}.
\end{align*}
\end{definition}
For convenience, we define the set of element-type pair that is feasible to a partial solution $\mathbf{s}$ as $\mathcal{F}(\mathbf{s}) = \{ (e,i) \in \left(V\setminus \supp(\mathbf{s})\right) \times [k] \mid \mathbf{s}+\mathbb{I}_{(e,i)} \in \mathcal{F} \}$.

The problem of maximizing a $k$-submodular function with fairness constraint can be formally formulated.
\begin{problem}[FM$k$SM]
\label{problem:fair}
    Given a monotone $k$-submodular function $f$, a total budget $B$, upper bounds $u_i$ and lower bounds $\ell_i$ for each type $i \in [k]$, Fair Monotone $k$-submodular Maximization problem (FMkSM) finds $\mathbf{s} \in \mathcal{F}$ such that $f(\mathbf{s})$ is maximized, i.e., 
    \[ \argmax_{\mathbf{s} \in \mathcal{F}} f(\mathbf{s}). \]
\end{problem}


\begin{remark}
    We assume $\sum_i^k u_i \geq B$; otherwise, Problem \ref{problem:fair} reduces to the individual size constrained problem \cite{ohsaka2015monotone}. We also assume $\mathcal{F}\neq \emptyset$, i.e., $\sum_i^k \ell_i \leq B$.
\end{remark}

For the sake of algorithm design, we call a solution $\mathbf{s}\in (k+1)^V$ {\em extendable} if $\mathbf{s} \in \mathcal{F}$. The following definition is derived from Definition \ref{def:feasible}, which can be used to efficiently verify if a solution is feasible or not. 
\begin{definition}
    \label{def:extendable}
     $\mathbf{s}$ is extendable if and only if
    \begin{align*}
        \forall i \in [k], |\supp_i(\mathbf{s})| \leq u_i
        \mbox{ and }  \sum_{i=1}^k \max\{|\supp_i(\mathbf{s})|, \ell_i\} \leq B.
    \end{align*}
\end{definition}

%% file: sections/3-algorithms.tex
\section{Fair Greedy Algorithms}
\label{sec:fair-greedy-monotone}
In this section, we present a greedy algorithm (Algorithm \ref{alg:fair}) to solve the $k$-submodular maximization with fairness constraint problem (Problem \ref{problem:fair}). We note that if the sum of the lower bounds is less than the total budget $B$, there is no feasible solution. If the sum of the upper bounds for each type is less than $B$, then Problem \ref{problem:fair} is reduced to the individual size constraint problem and the $\frac{1}{3}$-approximation algorithm by \cite{ohsaka2015monotone} can be used. 

Our greedy algorithm iteratively selects the best {\em extendable} element-type pair until the solution contains exactly $B$ elements. The candidate solutions should be {\em extendable} to ensure fairness. The extendable pairs in line 3 can be maintained efficiently by Definition \ref{def:extendable}. Our algorithm admits a $\frac{1}{3}$-approximation ratio and matches the best non-fair IS-constrained results by \cite{ohsaka2015monotone}. The theoretical results and analysis can be found in Theorem \ref{th:fair-greedy} below.

\begin{algorithm}[tb]
\caption{{\sc Fair-Greedy}}
\label{alg:fair}
\small
\textbf{Input}: Monotone $k$-submodular $f:(k+1)^V \rightarrow \mathbb{R}$, total budget $B$, upper bounds $u_i$ and lower bounds $\ell_i$ for $i \in [k]$.\\
\textbf{Output}: A fair solution $\mathbf{s} \in \mathcal{F}$.
\begin{algorithmic}[1] 
\STATE $\mathbf{s} \gets \mathbf{0}$.
\FOR{$j \gets 1$ \textbf{to} $B$} 
\STATE $\mathcal{I} \gets \{(e, i) \in V\setminus\supp(\mathbf{s}) \times [k] \mid \mathbf{s} + (e, i) \text{~is \textit{extendable}.} \}$
\STATE $(e,i) \gets \argmax_{(e',i')\in \mathcal{I}} \Delta_{e',i'}f(\mathbf{s})$
\STATE $\mathbf{s}(e) \gets i$
\ENDFOR
\STATE \textbf{return} $\mathbf{s}$
\end{algorithmic}
\end{algorithm}


\begin{theorem}
    Let $\mathbf{o}$ be the optimal solution and $\mathbf{s}$ be the output of Algorithm \ref{alg:fair}. The algorithm admits a $\frac{1}{3}$-approximation within $\mathcal{O}(knB)$ oracle evaluations to $f$ for Problem \ref{problem:fair} .
    \label{th:fair-greedy}
\end{theorem}

The running time is $\mathcal{O}(knB)$ since the number of iterations is bounded by $B$ and the size of $\mathcal{I}$ is at most $nk$.  The approximation analysis of the algorithm necessary for the proof of the above theorem relies on examining a specific way of altering the 
optimal solution $\mathbf{o}$ in each iteration based on the algorithm. The result of such alteration is denoted by $\mathbf{o}^j$ after iteration $j$. 
Finally, the relationship between $\mathbf{s}$ (result of the algorithm)
and the $\mathbf{o}^{B}$ is used to realize the approximation factor. The
strategy has been widely used in the analysis of $k$-submodular 
optimization problems~\cite{ohsaka2015monotone, zheng2021maximizing,iwata2016improved,ward2016maximizing,sakaue2017maximizing}.
However, it is worth noting that the existing proofs do not immediately
leads to the proof of the above theorem, primarily due to the fairness
constraint. 

\subsection{Construction of $\mathbf{o}^j$}

Due to the monotonicity of $f$, exactly $B$ element-type pairs can be selected by the
greedy solution at the end of the algorithm. The optimal solution size can also be
assumed to be $B$. If the optimal solution contains less than $B$ element-type pairs,
we can append feasible elements without decreasing the utility value. 
Similar to $\mathbf{o}^j$, $\mathbf{s}^j$ denotes the partial greedy solution as of iteration $j$ of  Algorithm \ref{alg:fair}.

Initially, $\mathbf{o}^0 = \mathbf{o}$ and $\mathbf{s}^0 = \mathbf{0}$.
Consider that $i\in [k]$,  $Q_i^j = \supp_i(\mathbf{o}^{j-1})\setminus \supp_i(\mathbf{s}^{j-1})$. That is, 
$Q_i^j$ is the set of elements that supports type $i$ in $\mathbf{o}_{j-1}$ 
and does not appear as supports for type $i$ in $\mathbf{s}^{j-1}$. Next,
we consider the different cases with respect to the element $e^j$ selected
to support the type $i^j$ in the greedy algorithm in iteration $j$.

\noindent
\textbf{Case 1 If $e^j \in Q_{i'}^j$ for some $i' \neq i^j$.}\\
In this case, the element $e^j$ is selected by the greedy algorithm and is present the optimal solution $\mathbf{o}$, but it is supporting different types. We have two sub-cases to consider.
\begin{itemize}
    \item \textbf{Case 1A}: If $Q_{i^j}^j \neq \emptyset$, we can identify an arbitrary element $q^j \in Q_{i^j}^j$.
    Then, we construct $\mathbf{o}^j$ as follows.
            
    $\mathbf{o}^{j-\frac{1}{2}}$ is a copy of $\mathbf{o}^{j-1}$ except that $(e^j, i')$ and $(q^j, i^j)$ are removed. $\mathbf{o}^{j}$ is constructed by adding $(e^j, i^j)$ and $(q^j, i')$ to 
    $\mathbf{o}^{j-\frac{1}{2}}$. Formally, 
    \[
        \begin{array}{rcl}
            \mathbf{o}^{j-\frac{1}{2}} & = & 
            \mathbf{o}^{j-1}-\mathbb{I}_{(e^j,i')} - \mathbb{I}_{(q^j,i^j)} \\
            \mathbf{o}^{j} & = & \mathbf{o}^{j-\frac{1}{2}}+\mathbb{I}_{(e^j,i^j)} + \mathbb{I}_{(q^j,i')}.
        \end{array}
    \]
            
    \item \textbf{Case 1B}: If $Q_{i^j}^j = \emptyset$, we can identify  an arbitrary element $q^j \in Q_{c}^j$ for some $c \in [k]$. We will prove the existence of $Q_c^j \neq \emptyset$ subsequently in Lemma \ref{lem:empty} below.
            
    We construct $\mathbf{o}^j$ as follows: 
    \[
        \mathbf{o}^{j-\frac{1}{2}}=\mathbf{o}^{j-1}- \mathbb{I}_{(q_j,c)} \mbox{ and  }\mathbf{o}^{j}=\mathbf{o}^{j-\frac{1}{2}}+\mathbb{I}_{(e^j,i^j)}.
    \]
\end{itemize}

\noindent
\textbf{Case 2: If $e^j \notin Q_{i'}^j$ for all $i' \neq i^j$.}\\
This case includes the mutually exclusive scenarios that $(e^j, i^j)$ is selected by the optimal solution $\mathbf{o}$ or $e^j$ is not in $\mathbf{o}$. Thus, we define $q^j$ as follows for this case.
\[q^j = \begin{cases}
    e^j  & \text{if } e^j \in Q_{i^j}^j \\
    \text{arbitrary element in } Q_{i^j}^j & \text{if } e^j \notin Q_{i^j}^j \text{ and } Q_{i^j}^j \neq \emptyset\\
    \text{arbitrary element in } Q_{c}^j   & \text{if } e^j \notin Q_{i^j}^j \text{ and } Q_{i^j}^j = \emptyset
\end{cases}\]
As mentioned before, We will prove the existence of $Q_c^j \neq \emptyset$ for the last case in Lemma \ref{lem:empty}.
    
In this case, 
\[
    \mathbf{o}^{j-\frac{1}{2}}=\mathbf{o}^{j-1}- \mathbb{I}_{(q_j,type(q_j))}
    \mbox{ and } \mathbf{o}^{j}=\mathbf{o}^{j-\frac{1}{2}}+\mathbb{I}_{(e_j,i^j)}.
\]

\subsection{Properties of $\mathbf{o}^j$ and $Q^j_i$}
    
    \begin{observation}
    \label{obs:sb-ob}
        For every iteration $j\in [B]$, we have $\mathbf{s}^{j-1} \prec \mathbf{o}^{j-\frac{1}{2}}$. At the end of iteration $B$,   $\mathbf{o}^B = \mathbf{s}^B$. 
    \end{observation}
    \begin{proof}
        At every iteration of the constructions of $\mathbf{o}^j$, the greedy solution $(e^j, i^j)$ was inserted in $\mathbf{o}^{j}$ and is never removed in subsequent
        iterations. Therefore, $\mathbf{o}^{j}$ contains every element of $\mathbf{s}^{j}$, i.e., $\mathbf{s}^{j} \prec \mathbf{o}^{j}$.
        Since $\mathbf{o}^{j-\frac{1}{2}}$ was constructed by only remove $(e^j, i')$ and $(q^j, i^j)$ or $(q^j, type(q^j))$ from $\mathbf{o}^{j-1}$ and these elements were not in $\mathbf{s}^{j-1}$ by the definition of $q^j$. Thus, we also have $\mathbf{s}^{j-1} \prec \mathbf{o}^{j-\frac{1}{2}}$.
        
        Since the $k$-submodular function $f$ is monotone and $\mathcal{F}\neq \emptyset$, there are $B$ elements in both $\mathbf{s}^B$ and $\mathbf{o}^B$. Thus, after $B$ iterations of constructions of $\mathbf{o}^j$, $\mathbf{o}^B = \mathbf{s}^B$.
    \end{proof}
    
    Due to the fairness constraint, $Q_i^j$ can be empty for some $i$. Therefore, we present the following lemma to address this scenario, which we have used in the construction of $\mathbf{o}^j$.
    
    \begin{lemma}
        \label{lem:empty}
        Let $(e^j, i^j)$ be the greedy choice at iteration $j$. If $Q_{i^j}^j = \emptyset$, then there must exist $c \in [k]$ such that $c \neq i^j$ and $Q_c^j \neq \emptyset$, and every element in $Q^j_c$ is extendable. 
    \end{lemma}
    \begin{proof}
    The proof relies on inductive argument on iterative construction of $\mathbf{o}^j$.
        Due to Observation \ref{obs:sb-ob} and the construction of $\mathbf{o}^{j}$, we have $\supp_i(\mathbf{s}^{j-1}) \subseteq \supp_i(\mathbf{o}^{j-1})$ for every $i \in [k]$,  By the definition of $Q^j_{i^j}$, if $Q_{i^j}^j =\emptyset$, then 
        \begin{equation}
            \label{eq:size-s-j1}
            |\supp_{i^j}(\mathbf{s}^{j-1})| = |\supp_{i^j}(\mathbf{o}^{j-1})|.
        \end{equation}
        
         A feasible optimal solution satisfies $|\supp(\mathbf{o})|=B$. Furthermore,
         based on the construction of $\mathbf{o}^{j}$ from $\mathbf{o}$ and
         the greedy selections, we have
        \begin{equation}
            \label{eq:size-o-j1}
            |\supp(\mathbf{s}^{j-1})|<|\supp(\mathbf{o}^{j-1})|=B \text{ for all } j \in [B]. 
        \end{equation}

        Therefore, from Eqs. \ref{eq:size-s-j1} and \ref{eq:size-o-j1}, there must exist some type $c$ such that $\supp_c(\mathbf{s}^{j-1}) \subset \supp_c(\mathbf{o}^{j-1})$, which implies that $Q_c^j \neq \emptyset$. Let $q^j$ be an arbitrary element from $Q_c^j$ with type $c$. Since $\mathbf{s}^{j-1}$ and $\mathbf{o}^{j-1}$ are feasible (partial) solutions, the following fairness constraint is satisfied: 
        \[ \ell_c \leq |\supp_c(\mathbf{s}^{j-1})| < |\supp_c(\mathbf{o}^{j-1})| \leq u_c. \]
        Therefore, considering $\mathbf{s}^{j-1} + \mathbb{I}_{(q^j, c)}$, the following extendability is also satisfied: $\ell_c \leq |\supp_c(\mathbf{s}^{j-1} + \mathbb{I}_{(q^j, c)})| \leq |\supp_c(\mathbf{o}^{j-1})| \leq u_c$, i.e., $(q^j, c) \in \mathcal{F}(\mathbf{s}^{j-1})$.
    \end{proof}

    \begin{lemma}
        \label{lem:s-o}
        For every iteration $j\in [B]$ 
        \[ f(\mathbf{s}^{j})-f(\mathbf{s}^{j-1}) \geq \frac{1}{2}(f(\mathbf{o}^{j-1}) -f(\mathbf{o}^{j})). \]
    \end{lemma}
    \begin{proof}

    We prove this lemma in three cases corresponding to the constructions of $\mathbf{o}^j$.
    
    \noindent
    \textbf{Case 1A}: $e^j \in Q_{i'}^j$ for some $i' \neq i^j$ and $Q_{i^j}^j \neq \emptyset$. 
    
    From the construction of $\mathbf{o}^j$ in Case 1A, we have
    \[
    \begin{array}{l}
    f(\mathbf{o}^{j-1}) = f(\mathbf{o}^{j-\frac{1}{2}}) + \Delta_{q^j, i^j}f(\mathbf{o}^{j-\frac{1}{2}}) 
      + \Delta_{e^j, i'}f(\mathbf{o}^{j-\frac{1}{2}}+\mathbb{I}_{(q^j, i^j)})
    \\[1em]
     f(\mathbf{o}^{j}) = f(\mathbf{o}^{j-\frac{1}{2}}) + \Delta_{q^j, i'}f(\mathbf{o}^{j-\frac{1}{2}}) 
     + \Delta_{e^j, i^j}f(\mathbf{o}^{j-\frac{1}{2}}+\mathbb{I}_{(q^j, i')})
    \end{array}
    \]
 
    Combining the above two inequalities, we have
    \begin{align}
        &f(\mathbf{o}^{j-1})-f(\mathbf{o}^{j}) \nonumber\\
        &\leq \Delta_{q^j, i^j} f(\mathbf{o}^{j-\frac{1}{2}})+\Delta_{e^j, i'} f(\mathbf{o}^{j-\frac{1}{2}}+\mathbb{I}_{(q^j, i^j)}) \nonumber\\
        &\leq \Delta_{q^j, i^j} f(\mathbf{o}^{j-\frac{1}{2}})+\Delta_{e^j, i'} f(\mathbf{o}^{j-\frac{1}{2}}) \quad\quad\mbox{submodularity} \nonumber\\
        &\leq \Delta_{q^j, i^j} f(\mathbf{s}^{j-1})+\Delta_{e^j, i'} f(\mathbf{s}^{j-1}) \quad\quad\mbox{Obs. \ref{obs:sb-ob} and submodularity} \nonumber\\
        &\leq \Delta_{e^j, i^j} f(\mathbf{s}^{j-1})+\Delta_{e^j, i^j} f(\mathbf{s}^{j-1})\quad\quad\mbox{greedy choice} \nonumber\\
        &=2(f(\mathbf{s}^{j})-f(\mathbf{s}^{j-1})).
        \label{eq:case1a}
    \end{align}

    \noindent\textbf{Case 1B: $e^j \in Q_{i'}^j$ for some $i' \neq i^j$ and $Q_{i^j}^j = \emptyset.$} 
    
    By the construction, we have 
    \[
    \begin{array}{l}
    f(\mathbf{o}^{j-1}) = f(\mathbf{o}^{j-\frac{1}{2}})+ \Delta_{q^j, i^j }f(\mathbf{o}^{j-\frac{1}{2}})\\[1em]
    f(\mathbf{o}^{j}) = f(\mathbf{o}^{j-\frac{1}{2}})+ \Delta_{e^j, i^j }f(\mathbf{o}^{j-\frac{1}{2}})
    \end{array}
    \] 
    Therefore, 
    \begin{align*}
        &f(\mathbf{o}^{j-1})-f(\mathbf{o}^{j})\\
        &=\Delta_{q^j, i^j}f(\mathbf{o}^{j-\frac{1}{2}}) - \Delta_{e^j, i^j }f(\mathbf{o}^{j-\frac{1}{2}}) \\
        &\leq \Delta_{q^j, i^j}f(\mathbf{o}^{j-\frac{1}{2}})\\
        &\leq \Delta_{q^j, i^j}f(\mathbf{s}^{j-1}) \quad\quad\mbox{Obs.~\ref{obs:sb-ob} and submodularity}\\
        &\leq \Delta_{e^j, i^j}f(\mathbf{s}^{j-1}) \quad\quad\mbox{Lemma \ref{lem:empty} and greedy choice}\\
        &=f(\mathbf{s}^{j})-f(\mathbf{s}^{j-1}) \leq 2(f(\mathbf{s}^{j})-f(\mathbf{s}^{j-1}))
    \end{align*}
    
    \noindent\textbf{Case 2: $e^j \notin Q_{i'}^j$ for all $i' \neq i^j$ and $Q_{i^j}^j \neq \emptyset.$} 
    
    The proof of this case is similar to Case 1B.
    \end{proof}

\subsection{Proof of Theorem \ref{th:fair-greedy}}
\label{proof:th1}
We are now ready to proceed with the proof of Theorem~\ref{th:fair-greedy}.
    \begin{align}
        &f(\mathbf{o})-f(\mathbf{s}) = f(\mathbf{o}^0)-f(\mathbf{s}^B) \nonumber\\
        &= f(\mathbf{o}^0)-f(\mathbf{o}^B) \quad\quad\quad\mbox{ due to Obs. 
        \ref{obs:sb-ob}} \nonumber\\
        &= \sum_{j=1}^B f(\mathbf{o}^{j-1}) - f(\mathbf{o}^{j}) \nonumber\\
        &\leq \sum_{j=1}^B 2(f(\mathbf{s}^{j}) - f(\mathbf{s}^{j-1})) \quad\quad\mbox{ due to Lemma \ref{lem:s-o} } \label{eq:proof-th1}\\
        &=2(f(\mathbf{s}^B)-f(\mathbf{s}^0)) \leq 2f(\mathbf{s}). \nonumber
    \end{align}
The $\frac{1}{3}$-approximation ratio is immediate by rearranging the above inequality. \hfill $\Box$

\subsection{Fair Approximate $k$-submodular Maximization.}
We discuss the scenario when Algorithm~\ref{alg:fair} only has approximate oracle access to the objective $k$-submodular function. The definition was presented in Section \ref{sec:prelim}. The approximation results are as follows.

\begin{theorem}
    
    Let $f$ be a $\delta$-approximate $k$-submodular function, Algorithm \ref{alg:fair} admits $\frac{1-\delta}{3+2B\left(\frac{1+\delta}{1-\delta}-1\right)}$-approximation within $\mathcal{O}(nkB)$ oracle evaluations. 
\label{th:approx-greedy}
\end{theorem}
\noindent
{\it Proof Sketch.} 
Due to space constraints, we present the proof ideas. It is based on the iterative construction of
$\mathbf{o}^j$  as in the proof of Theorem \ref{th:fair-greedy}. 
Assume the oracle $\Tilde{f}$ in Algorithm \ref{alg:fair} is $\delta$-approximate to the exact $k$-submodular $f$. 
Applying Definition \ref{def:approximate}, we have an inequality similar to Lemma \ref{lem:s-o}: $\frac{1+\delta}{1-\delta}f(\mathbf{s}^{j})-f(\mathbf{s}^{j-1}) \geq \frac{1}{2}(f(\mathbf{o}^{j-1}) -f(\mathbf{o}^{j}))$. The approximation ratio of Theorem \ref{th:approx-greedy} can be obtained by applying the modified lemma  
 in Section \ref{proof:th1} (the primary change is in Eq.~(\ref{eq:proof-th1})).
The approximation ratio of Theorem \ref{th:approx-greedy} reduces to Theorem \ref{th:fair-greedy} when $\delta=0$.

\section{A Faster Threshold Algorithm}
\label{sec:fast-fair-greedy}
Our greedy algorithm admits a $\frac{1}{3}$-approximation and requires $\mathcal{O}(nkB)$ oracle 
evaluations of the function $f$. This can be computationally expensive, being dependent
linearly on the parameter $B$, which can be prohibitively large. To improve the 
runtime while maintaining a good approximation, we propose a threshold algorithm
that requires only $\mathcal{O}(\frac{kn}
{\epsilon}\log\frac{B}{\epsilon})$ evaluations of $f$. Our algorithm incorporates the ideas of lazy evaluation and threshold technique 
\cite{badanidiyuru2014fast}, making it faster than the lazy evaluation implementations 
of Algorithm \ref{alg:fair}. The theoretical results and analysis are 
presented in Theorem \ref{th:fair-faster}.

In Algorithm~\ref{alg:fair-fast}, we use a priority queue over the element-type pairs where the priority is its marginal gain, initially described in terms of the function valuation of the element-type pair. 
The queue is refined to remove any element-type pair whose addition to the partial solution $\mathbf{s}^{j-1}$ ($\mathbf{s}^0 = 0$) is not extendable (line~5). If the priority queue is empty, the iterative algorithm terminates, outputting the partial solution (line~7). Otherwise, the highest priority element-type pair $(e^j, i^j)$ is considered to be added to the partial solution. 
Each element-type pair is also associated with a counter $u$, which is initially $0$ and is incremented every time the element-type pair appears at the top of the priority queue. 
If the addition of the element-type pair results in a marginal gain which is at least $(1-\epsilon)$ factor of the priority value associated with the element-pair, then the element-type pair is good enough to be added to the partial solution (lines~11-12). Otherwise, the element-type pair is re-inserted to the priority queue with an updated priority value equal to the marginal gain it could have realized to the partial solution at that iteration. Due to submodularity, an evaluation of an element-type pair will result in the decrease of its corresponding priority in the queue. Therefore, an element-type pair is allowed to be re-inserted (re-evaluated) for at most $\frac{1}{\epsilon}\log \frac{B}{2\epsilon}$ times (lines~16-17).

\begin{algorithm}[tb]
\caption{\sc Fair-Threshold}
\label{alg:fair-fast}
\small
\textbf{Input}: Monotone $k$-submodular $f:(k+1)^V \rightarrow \mathbb{R}_{\geq 0}$, total budget $B$, upper bounds $u_i$ and lower bounds $\ell_i$ for all $i \in [k]$, error threshold $\epsilon \in (0,1)$.\\
\textbf{Output}: A fair solution $\mathbf{s} \in \mathcal{F}$.
\begin{algorithmic}[1] 
\STATE $\mathbf{s}^0 \gets \mathbf{0}, j\gets 1$
\STATE $PQ \gets$ priorityQueue$(\{(f(\mathbb{I}_{(e,i)}), (e,i))~|~(e,i) \in V\! \times\! [k]\})$\\
\STATE $u((e,i)=0$ for all $(e,i) \in V\! \times\! [k]$\\
\WHILE{$j \leq B$}
\STATE    Remove entries from $PQ$ with $(e,i)$ such that $e \in \supp(\mathbf{s}^{j-1})$ or $\mathbf{s}^{j-1}+\mathbb{I}_{(e,i)}$ is not \textit{extendable}.
        \IF{$PQ$ is empty}
         \STATE   Return $\mathbf{s}^{j-1}$ 
        \ENDIF
        \STATE $(\tau_j, (e^j, i^j)) \gets PQ.getTopAndRemove()$ 
        \STATE $u((e^j,i^j))++$
        \IF {$\Delta_{e^j, i^j}f(\mathbf{s}^{j-1}) \geq (1-\epsilon)\tau_j$}
          \STATE  $\mathbf{s}^{j} \gets \mathbf{s}^{j-1}+\mathbb{I}_{(e^j,i^j)}$
          \STATE  $j++$
           \STATE Continue to the Next Iteration 
        \ENDIF
        \IF{$u((e^j, i^j)) \leq \frac{1}{\epsilon}\log \frac{B}{2\epsilon}$}
          \STATE  $PQ.push((\Delta_{e^j, i^j}f(\mathbf{s}^{j-1}), (e^j,i^j))$) 
        \ENDIF
    \ENDWHILE
\STATE \textbf{return} $\mathbf{s}$
\end{algorithmic}
\end{algorithm}

\begin{theorem}
    \label{th:fair-faster}
    For Problem \ref{problem:fair}, Algorithm \ref{alg:fair-fast} outputs a 
    ($\nicefrac{1}{3}-\epsilon$)-approximation solution within $\mathcal{O}(\frac{kn}{\epsilon}\log\frac{B}{\epsilon})$ oracle evaluations to $f$, where $\epsilon \in (0,1)$ is the error threshold.
\end{theorem}

Once again, we will consider conceptually the the constructions of $\mathbf{o}^j$ as we have done for the proof of Theorem \ref{th:fair-greedy}. 
In the approximation analysis, we first assume that exactly $B$ elements can be selected by the algorithm and derive a $\nicefrac{1}{3} -\epsilon$ approximation guarantee. However, observe that Algorithm \ref{alg:fair-fast} can terminate early with a solution size less than $B$, due to the cases where element-type pairs may be discarded at line~16. We will justify that such eliminations impact the partial solution trivially enough, and hence, the $\nicefrac{1}{3}-\epsilon$ approximation can still be guaranteed.

    Using the same construction of $\mathbf{o}^{j}$, we can derive the following lemma for Algorithm \ref{alg:fair-fast}.
    \begin{lemma}
        \label{lem:s-o-2}
        Let $\ell$ be the size of the solution acquired at the end of Algorithm \ref{alg:fair-fast}. For every $j\in\{1, \cdots, \ell\}$,
        \[ f(\mathbf{s}^{j})-f(\mathbf{s}^{j-1}) \geq \frac{1-\epsilon}{2}\left(f(\mathbf{o}^{j-1}) -f(\mathbf{o}^{j})\right). \]
    \end{lemma}
    \begin{proof}
    Note that, if the element-type pair $(e^j, i^j)$ is added to the partial solution 
    $\mathbf{s}^{j-1}$, the following condition is satisfied
    \[ \Delta_{e^j, i^j} f(\mathbf{s}^{j-1}) \geq (1-\epsilon)\tau_j.\]

    Furthermore, for every $e \notin \supp(\mathbf{s}^{j-1}) \cup \{ e^{j} \}$, such
    that $\forall i \in [k], (e,i) \in \mathcal{F}(\mathbf{s}^{j-1})$ (as guaranteed 
    in line 5) are not at the top of the priority queue and their marginal gains evaluated in all previous iterations $b \in [0, j-1)$ is $\Delta_{e, i} f(\mathbf{s}^{b}) < \tau_j.$
    
    We know that $\mathbf{s}^{b} \prec \mathbf{s}^{j-1}$ and hence, due to submodularity, we have 
    \[ \Delta_{e, i} f(\mathbf{s}^{j-1}) \leq \Delta_{e, i} f(\mathbf{s}^{b}) < \tau_j. \]

    Therefore, the following relation holds for every $(e,i) \in  \left(V\setminus \supp(\mathbf{s}^{j-1}) \cup \{ e^{j} \}\right) \times [k]$, 
    \begin{equation}
        \Delta_{e^{j}, i^{j}} f(\mathbf{s}^{j-1}) \geq (1-\epsilon) \Delta_{e, i} f(\mathbf{s}^{j-1}).   
        \label{eq:threshold}
    \end{equation}
    Similar to the proof of Lemma \ref{lem:s-o}, we have three cases. 

    \noindent {\bf Case 1A}: The formulations of $f(\mathbf{o}^j)$ and $f(\mathbf{o}^{j-1})$ can be found in Case 1A of Theorem \ref{th:fair-greedy} proof.
    Therefore, the rest proof of this lemma follows the derivations of Eq. (\ref{eq:case1a}) with incorporation of Eq. (\ref{eq:threshold}).
    \begin{align*}
        &f(\mathbf{o}^{j-1})-f(\mathbf{o}^{j}) \\
        &\leq \Delta_{q^j, i^j} f(\mathbf{s}^{j-1})+\Delta_{e^j, i'} f(\mathbf{s}^{j-1}) \\
        &\leq \frac{1}{1-\epsilon}\Delta_{e^j, i^j} f(\mathbf{s}^{j-1})+\frac{1}{1-\epsilon}\Delta_{e^j, i^j} f(\mathbf{s}^{j-1})~~\mbox{due to Eq. (\ref{eq:threshold})}\\
        &=\frac{2}{1-\epsilon}(f(\mathbf{s}^{j})-f(\mathbf{s}^{j-1})).
    \end{align*}
    The Cases 1B and 2 proofs of Lemma \ref{lem:s-o-2}  can be derived similarly to Case 1A shown above.
    \end{proof}

    \subsection{Proof of Theorem \ref{th:fair-faster}}
    \label{proof:th2}
    Since the constructions of $\mathbf{o}^j$ are the same as in Theorem \ref{th:fair-greedy}, Observation  \ref{obs:sb-ob} still holds. For Theorem \ref{th:fair-faster}, we consider two cases related
    to the size of the final solution.
    
    \noindent\textbf{Case 1: $|\supp(\mathbf{s})|=B$.\ }
    In this case, the derivations are similar to the $\frac{1}{3}$-approximation proof in Theorem \ref{th:fair-greedy}.
    \begin{align*}
        &f(\mathbf{o}) - f(\mathbf{s}) = f(\mathbf{o}^0) - f(\mathbf{o}^B)= \sum_{j=1}^B \Big [ f(\mathbf{o}^{j-1}) - f(\mathbf{o}^j) \Big ]\\
        &\leq \frac{2}{1-\epsilon}\sum_{j=1}^B \Big [f(\mathbf{s}^{j}) - f(\mathbf{s}^{j-1})\Big ]\quad\quad\mbox{due to Lemma \ref{lem:s-o-2}}\\
        &=\frac{2}{1-\epsilon}f(\mathbf{s}).
    \end{align*}
    Rearranging the above inequality, we have
    \[ f(\mathbf{s}) \geq \frac{1}{ \frac{2}{1-\epsilon}+1} f(\mathbf{o}) = \frac{1-\epsilon}{3-\epsilon}f(\mathbf{o}) \geq \left(\frac{1}{3}-\epsilon\right) f(\mathbf{o}).\]
    
    \noindent\textbf{Case 2: $|\supp(\mathbf{s})|<B$.\ }
    We denote $\Tilde{\mathbf{s}}$ as solution with \textit{exactly} $B$ elements if there was no line 16 in Algorithm~\ref{alg:fair-fast}, and $|\supp(\mathbf{s})| = \ell$. The elements $e \in \Tilde{\mathbf{s}} \setminus \mathbf{s}$ are not added due to line 16 of the algorithm, which has been considered (as they are at the top of the priority queue) for $\frac{1}{\epsilon}\log\frac{B}{2\epsilon}$ times. Therefore, for every additional pair $(e,i)$ such that $e \in \mathbf{\Tilde{s}} - \mathbf{s}$, we have:
    \begin{equation}
    \label{eq:trivial}
        \Delta_{e,i} f(\mathbf{s}) < (1-\epsilon)^{\frac{\log B/2\epsilon}{\epsilon}}f(\mathbb{I}_{(e,i)}) \leq \frac{2\epsilon}{B}f(\mathbb{I}_{(e,i)}). 
    \end{equation}
    
    Next, we bound the utility value of the elements not added to the final solution due to line 16.
    \begin{align*}
        f(\Tilde{\mathbf{s}}) - f(\mathbf{s}) &\leq \sum_{(e,i) \in \Tilde{\mathbf{s}} - \mathbf{s}} \Delta_{e,i} f(\mathbf{s})\\
        &\leq (B-\ell) \frac{2\epsilon}{B}f(\mathbb{I}_{(e,i)}) \quad\quad\mbox{due to Eq. (\ref{eq:trivial})}\\
        &\leq (B-\ell) \frac{2\epsilon }{3B} f(\mathbf{o})\quad\quad\mbox{assume $f(e)\leq \frac{1}{3}f(\mathbf{o})$}\\
        &\leq  \frac{2\epsilon}{3}f(\mathbf{o}).
    \end{align*}
    Rearranging the above inequality, and since $\Tilde{\mathbf{s}}$ is a $\frac{1-\epsilon}{3-\epsilon}$-approximation solution, we have
    \begin{align*}
        f(\mathbf{s}) \geq f(\Tilde{\mathbf{s}}) -  \frac{2\epsilon}{3}f(\mathbf{o})
        &\geq \frac{1-\epsilon}{3-\epsilon}f(\mathbf{o}) -  \frac{2\epsilon}{3}f(\mathbf{o})\\
        &= \left(\frac{1}{3}-\epsilon\right)f(\mathbf{o}).
    \end{align*}

Cases 1 and 2 conclude the proof of $(\frac{1}{3}-\epsilon)$-approximation ratio of Theorem \ref{th:fair-faster}.

\begin{lemma}
    \label{lem:runtime-faster}
    The number of oracle evaluations of Algorithm \ref{alg:fair-fast} is bounded by $\mathcal{O}(\frac{kn}{\epsilon}\log\frac{B}{\epsilon})$.
\end{lemma}
\begin{proof}
The size of the priority queue is at most $nk$ by line 2 of Algorithm \ref{alg:fair-fast}.
As per line 16, an element-type in the priority queue can be evaluated for at most $\frac{1}{\epsilon}\log\frac{B}{2\epsilon}$ times. Thus, the total number oracle evaluations is at most $\frac{kn}{\epsilon}\log\frac{B}{2\epsilon} \in \mathcal{O}(\frac{kn}{\epsilon}\log\frac{B}{\epsilon})$. 
\end{proof}


\subsection{Fair Approximate $k$-submodular Maximization.}
Our threshold algorithm (Algorithm \ref{alg:fair-fast}) can also handle approximate oracles. We present the  approximation results are as follows. 

\begin{theorem}
    Let $f$ be a $\delta$-approximate $k$-submodular function, Algorithm \ref{alg:fair-fast} admits $\frac{1-2\epsilon(1+\delta)}{1+2\delta B}\cdot \alpha$-approximation within $\mathcal{O}(\frac{kn}{\epsilon}\log\frac{B}{\epsilon})$ oracle evaluations, where $\alpha = \frac{1-\delta}{3+2B\left(\frac{1+\delta}{(1-\delta)(1-\epsilon)}-1\right)}$, $\epsilon \in (0,1)$ and $\delta \in [0,1)$. 
    \label{th:approx-thresh}
\end{theorem}
\noindent
{\it Proof Sketch.} 
Similar to Theorem \ref{th:fair-faster}, $\mathbf{o}^j$ are constructed in a same way. The proofs of this theorem use Definition \ref{def:approximate} and modified Lemma \ref{lem:s-o-2}.
Assume the oracle $\Tilde{f}$ in Algorithm \ref{alg:fair-fast} is $\delta$-approximate to the exact $k$-submodular $f$. 
Applying Definition \ref{def:approximate}, we have an inequality similar to Lemma \ref{lem:s-o-2}: $\frac{1+\delta}{1-\delta}f(\mathbf{s}^{j})-f(\mathbf{s}^{j-1}) \geq \frac{1-\epsilon}{2}(f(\mathbf{o}^{j-1}) -f(\mathbf{o}^{j}))$. The approximation ratio of Theorem \ref{th:approx-thresh} can be obtained by following the steps in Section \ref{proof:th2}.

%% file: sections/4-experiments.tex
\section{Experiments}
\label{sec:experiements}


\subsection{Algorithms Compared}

\noindent
$\bullet$ \textsc{Fair-Greedy} (Algorithm~\ref{alg:fair}) and \textsc{Fair-Threshold} (Algorithm~\ref{alg:fair-fast})~\footnote{The code can be found: \url{ https://github.com/yz24/fair-k-submodular}}: there are total budget constraint $B$ and fairness constraint (both upper bounds $u_i$ and lower bounds $\ell_i$ for each type $i\in [k]$).

\noindent
$\bullet$ \textsc{TS-greedy} \cite{ohsaka2015monotone}: Total size constrained greedy algorithm with a $\frac{1}{2}$ approximation. The total size constraint is $B$. The results of this algorithm may violate the fairness constraint (both lower and upper bounds).

\noindent
$\bullet$ \textsc{IS-greedy} \cite{ohsaka2015monotone}: Individual size constrained greedy algorithm with a $\frac{1}{3}$ approximation. The individual size constraints (upper bounds) are the same as $u_i$ for each type $i \in [k]$. We give an additional total budget $B$ to this algorithm. The solutions from this algorithm may violate the lower bounds of the fairness constraint. 

For all the above algorithms, we also consider algorithms with $\delta$-approximate oracles \cite{zheng2021maximizing}. 
\textsc{TS-greedy} realizes $\frac{(1-\delta)^2}{2(1-\delta+B\delta)(1+\delta)}$ approximation for approximate oracles while \textsc{IS-greedy} realizes
$\frac{(1-\delta)^2}{(3-3\delta+2B\delta)(1+\delta)}$ approximation.

\noindent
$\bullet$ \textsc{random-fair}: A random feasible fair solution set of $B$ elements satisfying the fairness constraint. We run this procedure ten times and evaluate the expected objective value of a random fair solution.


\subsection{Evaluation Metrics} 

\noindent$\bullet$ \textbf{Objective values}. The objective values are the expected number of influence spread for the application of fair influence maximization with $k$ topics (see Section \ref{sec:influence}) and the entropy for the application of fair sensor placement with $k$ types (see Section \ref{sec:sensor}).
    
\noindent$\bullet$ \textbf{Oracle evaluations}. All algorithms considered are designed under the standard value oracle model, which assumes that an exact or an approximate oracle can return the value $f(S)$ when provided with a set $S \subseteq V$. Therefore, for comparisons, we report the number of evaluations of $f$ as a metric of time complexity.

\noindent$\bullet$ \textbf{Bias Error}. Since the baselines have no fairness constraint, they can produce biased solutions. Thus, we use an error metric \cite{el2020fairness} to quantify the unfairness of a solution $\mathbf{s}$: $\text{err}(\mathbf{s}) = \max_{i \in [k]} \{|\supp_i(\mathbf{s})|-u_i,\ell_i-|\supp_i(\mathbf{s})|,0\}$. The error measures the number of elements in a solution that violates the fairness constraint. The error takes value from $[0,2B]$; the higher the value is, the more biased a solution is.

\subsection{Fair Influence Maximization with $k$ Topics}

\label{sec:influence}
The $k$-topic independent cascade ($k$-IC) model \cite{ohsaka2015monotone} generalizes the independent cascade model \cite{kempe2003maximizing} in influence diffusion. 
Given a social network $G=(V, E)$, each edge $(u\in V, v\in V) \in E$ is associated with weights 
$\{p_{u, v}^i\}_{i \in[k]}$, where $p_{u, v}^i$ represents the strength of influence from user $u$ to user $v$ on topic $i$. The influence spread $\sigma:(k+1)^V \rightarrow \mathbb{R}^{+}$ in the $k$-IC model is defined as the expected number of combined influenced nodes by a seed set. Formally, $\sigma(\mathbf{s})=\mathbb{E}\left[\left|\bigcup_{i \in[k]} A_i\left(\supp_i(\mathbf{s})\right)\right|\right]$, where $A_i\left(\supp_i(\mathbf{s})\right)$ is a random variable representing the set of influenced nodes in the diffusion process of the $i$-th topic. 
The fair $k$-submodular maximization (Problem \ref{problem:fair}) is to select a seed set $\mathbf{s} \in (k+1)^V$ such that $\argmax_{\mathbf{s} \in \mathcal{F}} \sigma(\mathbf{s}) $

The influence spread function $\sigma$ is monotone $k$-submodular~\cite{ohsaka2015monotone}. When the influence spread of a solution $\mathbf{s} \in (k+1)^V$ is not accessible but can be approximated within a $\delta$ error, i.e., the approximate influence spread $\Tilde{\sigma}(\mathbf{s})$ satisfies $(1-\delta)\sigma(\mathbf{s}) \leq \Tilde{\sigma}(\mathbf{s}) \leq (1+\delta)\sigma(\mathbf{s})$ for some $\delta \in [0,1)$.

\subsubsection{\bf Experiment settings}  
We use the preprocessed data of the Digg network with 3,523 users, 90,244 links, and $k=10$ topics~\cite{ohsaka2015monotone}. 
For the fairness constraint, we set the range of the number of users to be selected for each topic as $[3, 10]$. The total budget varies from $\{30,40,50,60,70,80,90,100\}$. 
The influence spread $\sigma(\cdot)$ is estimated by simulating the diffusion process 100 times. When the algorithms terminated, we simulated the diffusion process 1,000 times to obtain sufficiently accurate estimates of the spread. 

For the experiment of $\delta$-approximate $k$-submodular function cases, we simulate the approximate information spread by outputting a random number from $[(1-\delta)\sigma(\mathbf{s}), (1+\delta)\sigma(\mathbf{s})]$ for a given error $\delta \in [0,1)$. In our experiment, $\delta$ is set to be $0.1$ and $0.2$. Since the oracle values are drawn randomly, we run the algorithms for $\lceil \log |V| \rceil$ times and report the median of the oracle value of the $\lceil \log |V| \rceil$ solutions.

\subsubsection{\bf Results} 
The objective value comparison results of the influence maximization with $k$ topics are shown in Figs. \ref{fig:digg-1}, \ref{fig:digg-1-delta1} and \ref{fig:digg-1-delta2}. The oracle evaluation comparisons are presented in Figs. \ref{fig:digg-2}, \ref{fig:digg-2-delta1} and \ref{fig:digg-2-delta2}.
The errors for each budget of the baseline algorithms are illustrated in the top parts of Figs. \ref{fig:errors},  \ref{fig:errors-1} and \ref{fig:errors-2}.

From Fig. \ref{fig:digg-1}, we can see that {\sc TS-greedy} and {\sc Fair-Greedy} produce the highest objective values. Our faster algorithm {\sc Fair-Threshold} is also effective even with a high error threshold. On the other hand, {\sc TS-greedy} and {\sc Fair-Greedy} need the most oracle evaluations, as shown in Fig. \ref{fig:digg-2} and {\sc TS-greedy} will violate the fairness constraint when $B=80,90$ and $100$. However, our faster version {\sc Fair-Threshold} with $\epsilon=0.5$ has a big improvement regarding oracle evaluations. 

For the experiments with $\delta$-approximate oracles ($\delta \neq 0$), the objective values are lower compared to exact oracle settings. Nonetheless, our algorithms produce similar quality solutions to the baseline algorithms. The number of oracle evaluations does not strictly increase as the budget increases, which is due to the lazy evaluation implementations. In terms of the bias errors,  the {\sc TS-greedy} algorithm tends to add elements from one single topic, resulting in high error for $\delta=0.1$ and $0.2$. This is expected as the algorithms are not focused on 
conforming to the fairness constraints. 

\begin{figure*}[t]
\centering
\subfigure[Exact Oracles ($\delta=0$)]{
\centering
\includegraphics[width=0.26\textwidth]{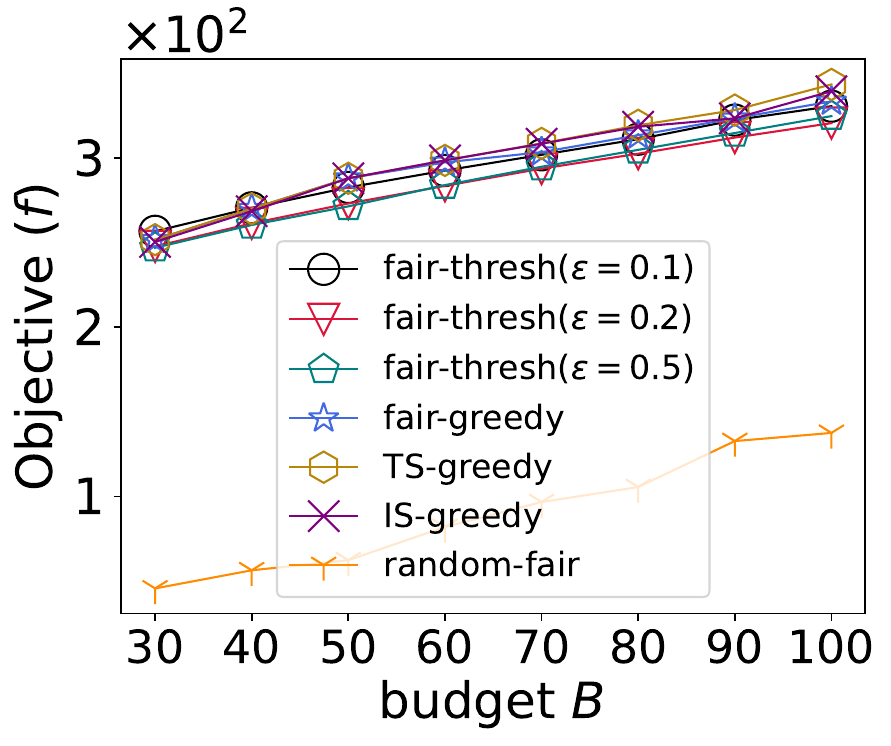}
\label{fig:digg-1}
}
~
\subfigure[Approximate Oracles ($\delta=0.1$)]{
\centering
\includegraphics[width=0.27\textwidth]{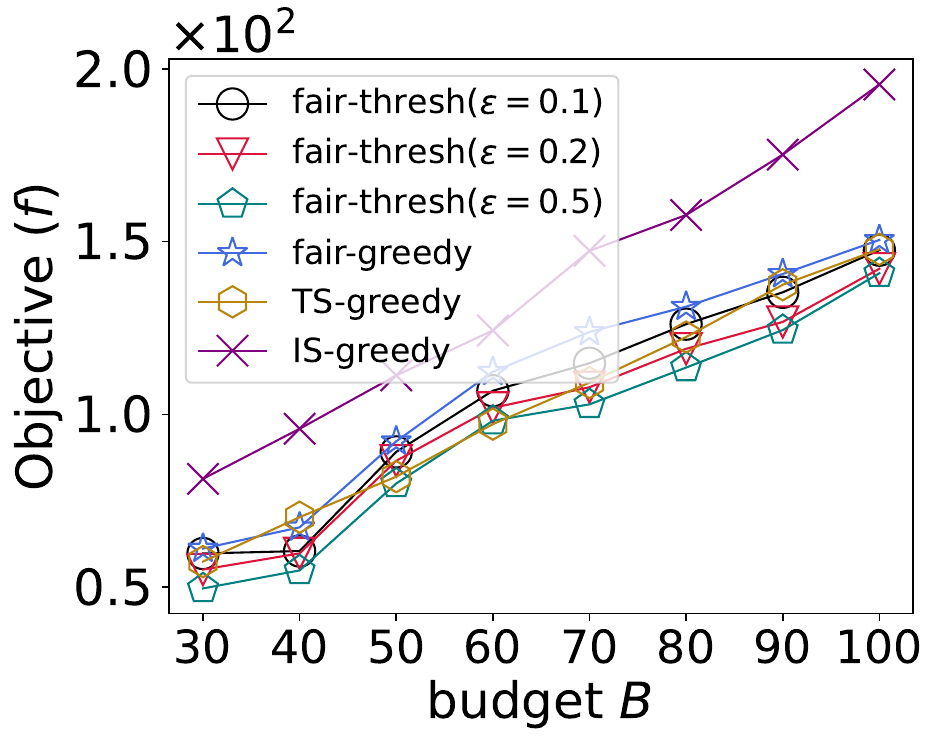}
\label{fig:digg-1-delta1}
}
~
\subfigure[Approximate Oracles ($\delta=0.2$)]{
\centering
\includegraphics[width=0.27\textwidth]{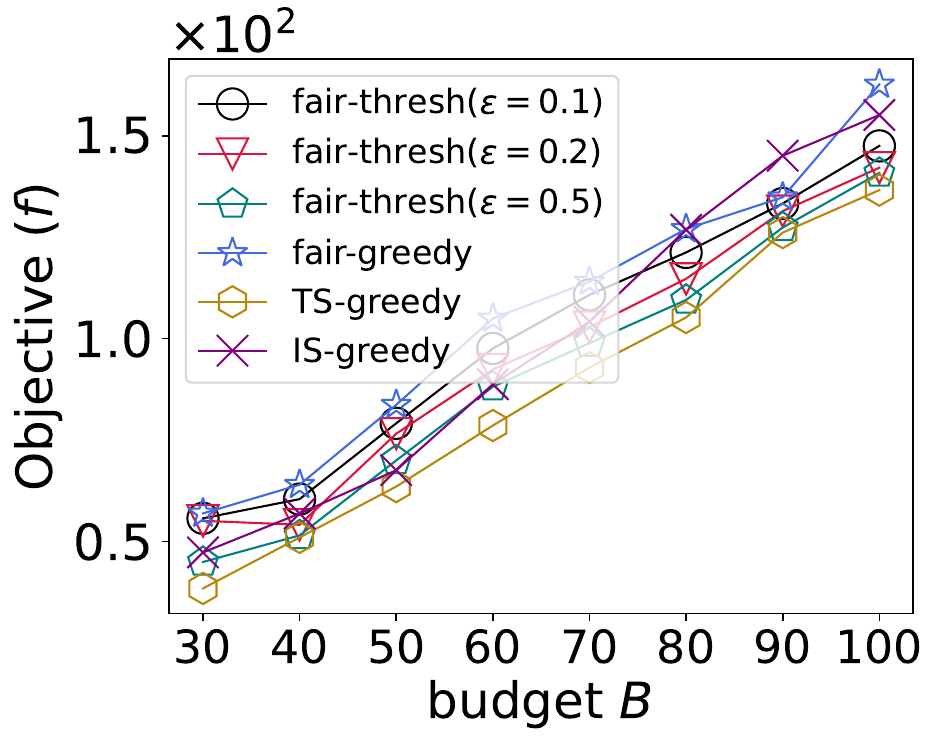}
\label{fig:digg-1-delta2}
}
\caption{Experiments on influence maximization with $k$ topics on Digg network: budget $B$ vs. objective values (expected influence spread).}
\end{figure*}

\begin{figure*}[t]
\centering
\subfigure[Exact Oracles ($\delta=0$)]{
\centering
\includegraphics[width=0.27\textwidth]{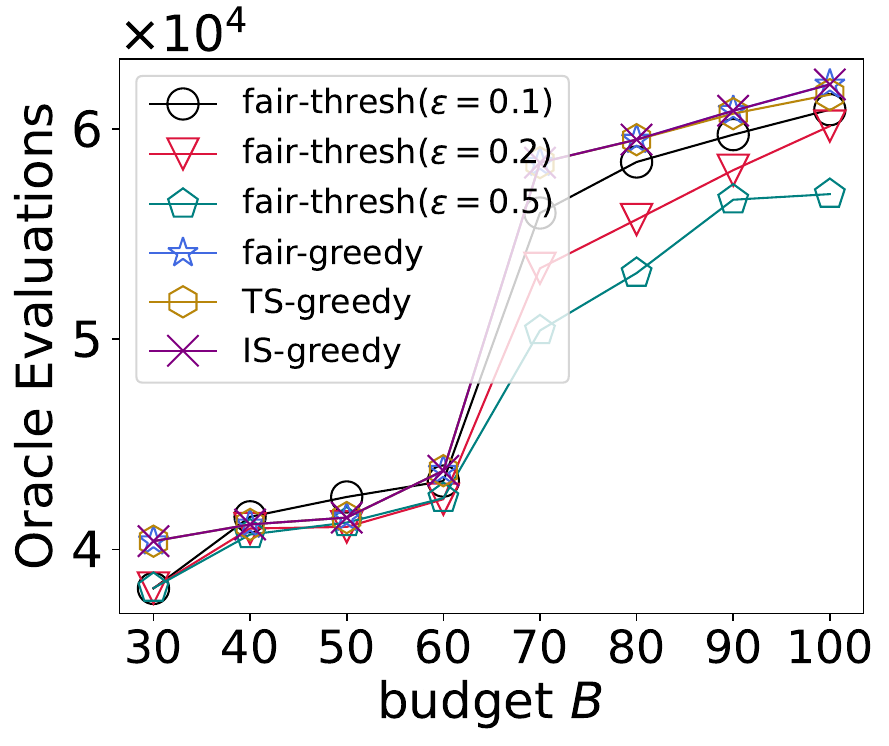}
\label{fig:digg-2}
}
~
\subfigure[Approximate Oracles ($\delta=0.1$)]{
\centering
\includegraphics[width=0.28\textwidth]{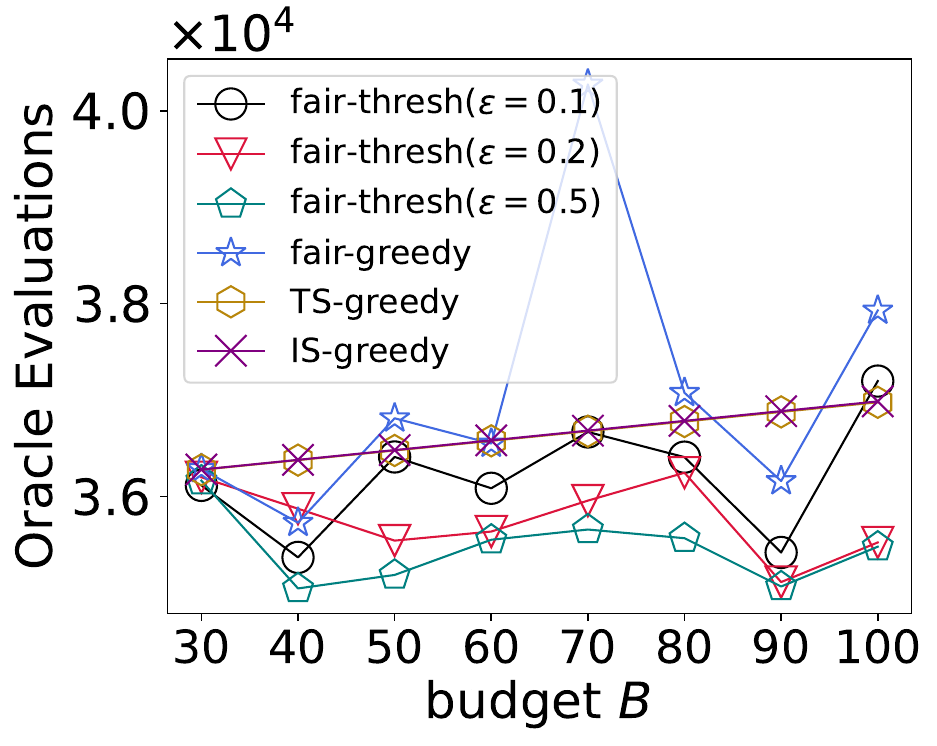}
\label{fig:digg-2-delta1}
}
~
\subfigure[Approximate Oracles ($\delta=0.2$)]{
\centering
\includegraphics[width=0.28\textwidth]{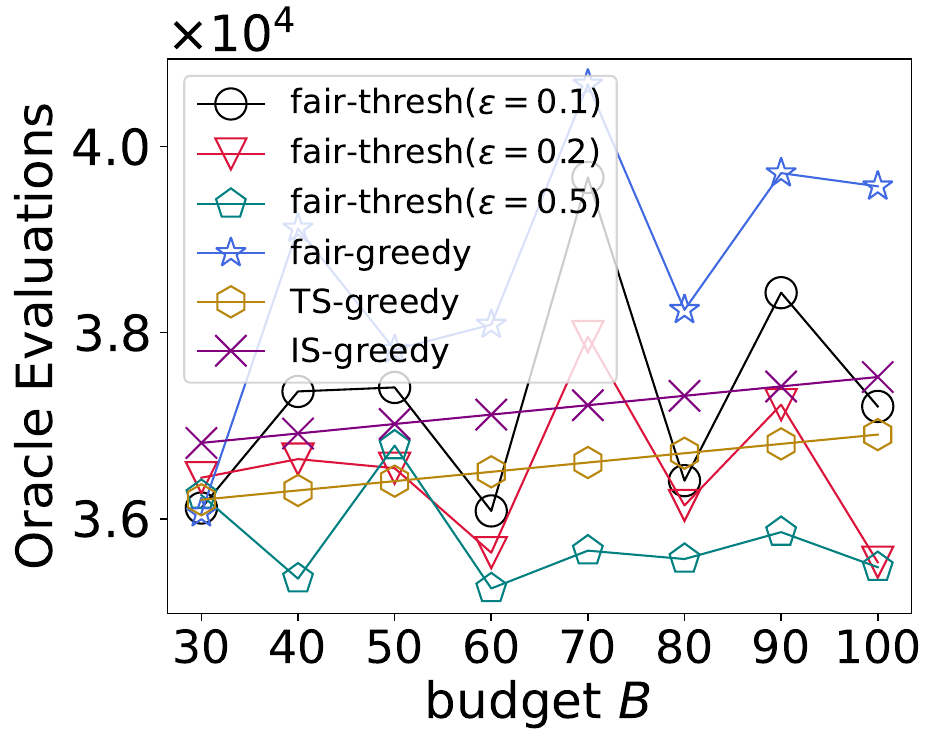}
\label{fig:digg-2-delta2}
}
\caption{Experiments on influence maximization with $k$ topics on Digg network: budget $B$ vs.  Oracle Evaluations.}
\end{figure*}


\begin{figure*}[t]
\centering
\subfigure[Exact Oracles ($\delta=0$)]{
\centering
\includegraphics[width=0.27\textwidth]{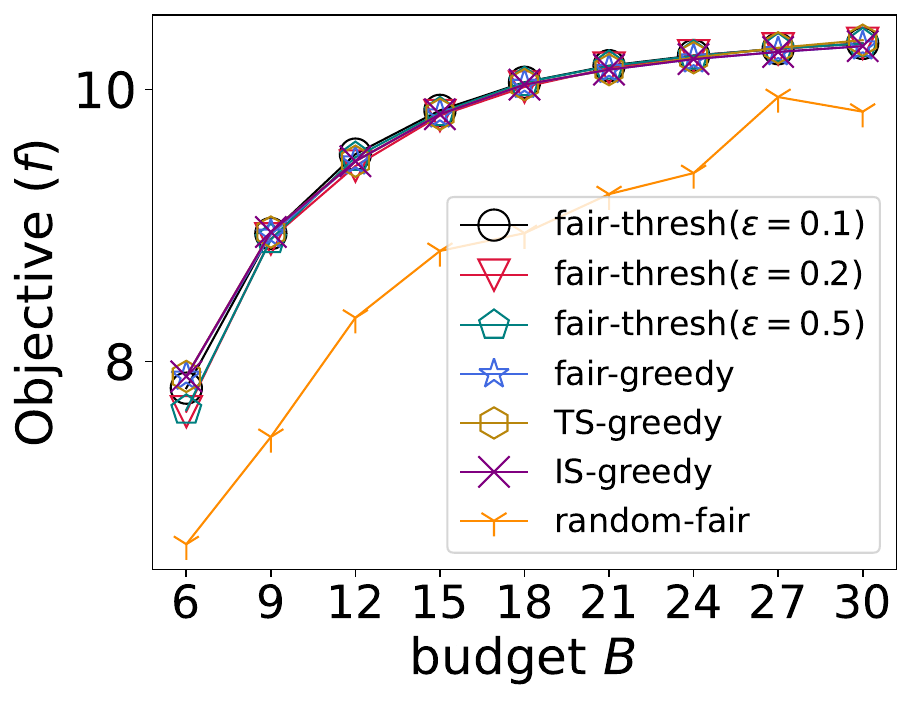}
\label{fig:sensor-1}
}
~
\subfigure[Approximate Oracles ($\delta=0.1$)]{
\centering
\includegraphics[width=0.27\textwidth]{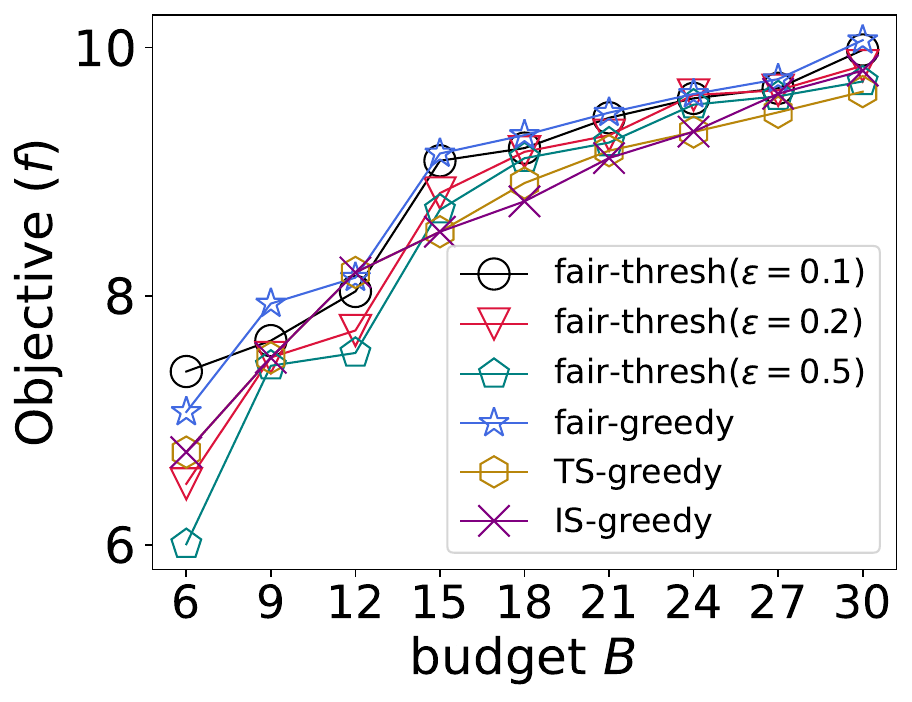}
\label{fig:sensor-1-delta1}
}
~
\subfigure[Approximate Oracles ($\delta=0.2$)]{
\centering
\includegraphics[width=0.27\textwidth]{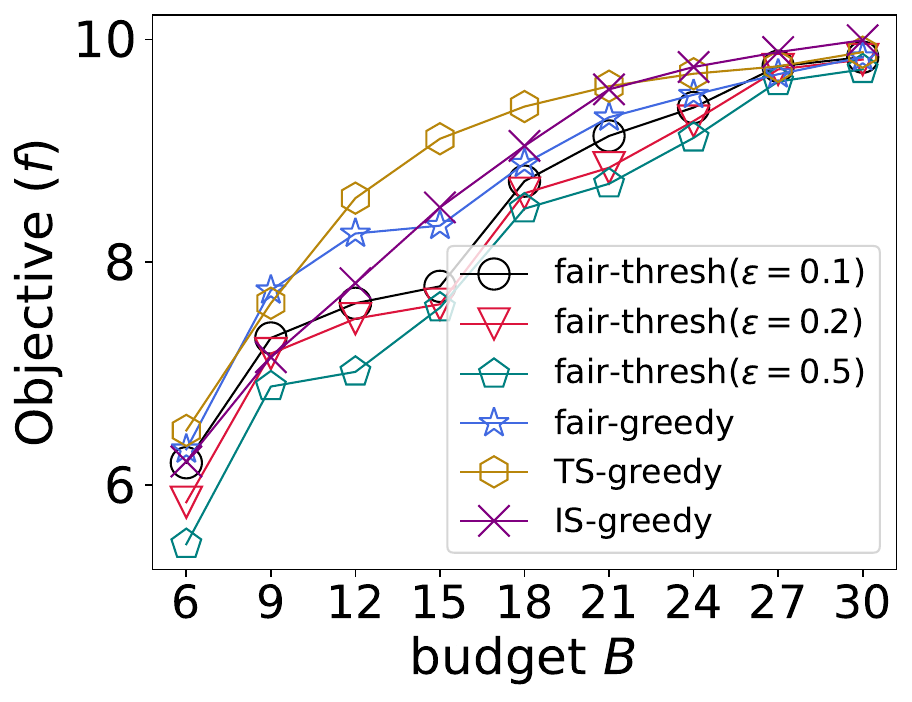}
\label{fig:sensor-1-delta2}
}
\caption{Experiments on sensor placement with $k$ types on Intel Lab sensor data: budget $B$ vs. objective values (entropy).}
\end{figure*}

\begin{figure*}[t]
\centering
\subfigure[Exact Oracles ($\delta=0$)]{
\centering
\includegraphics[width=0.28\textwidth]{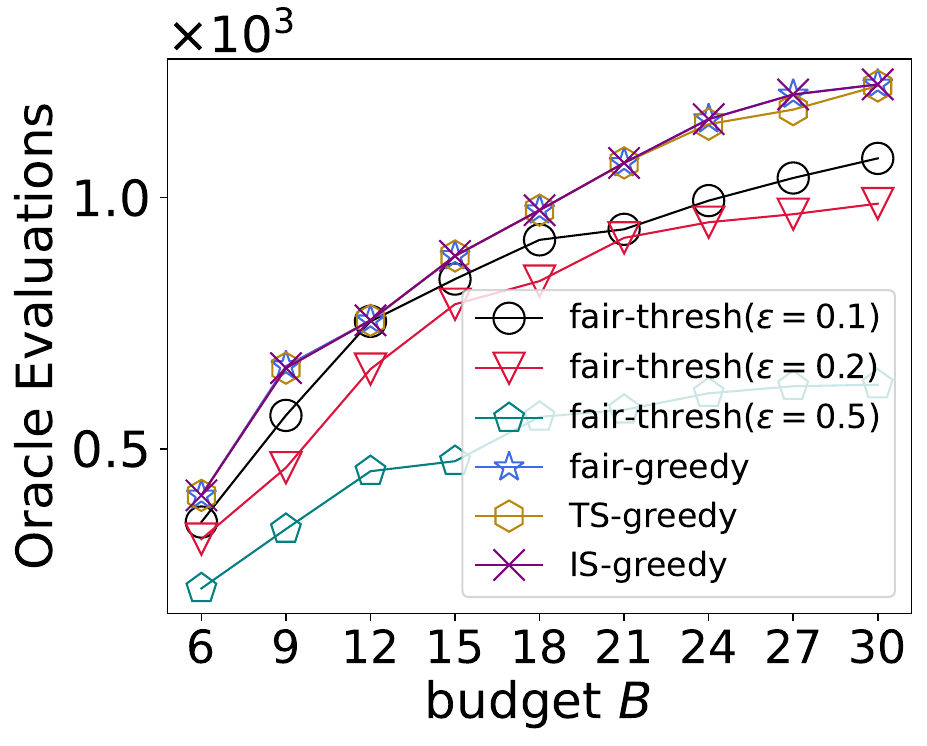}
\label{fig:sensor-2}
}
~
\subfigure[Approximate Oracles ($\delta=0.1$)]{
\centering
\includegraphics[width=0.27\textwidth]{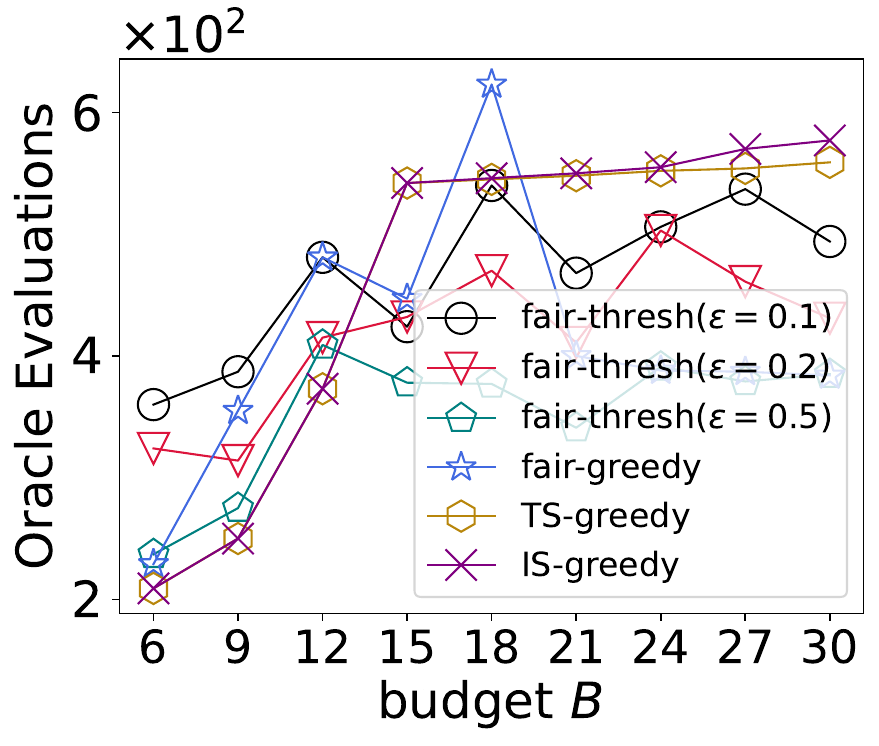}
\label{fig:sensor-2-delta1}
}
~
\subfigure[Approximate Oracles ($\delta=0.2$)]{
\centering
\includegraphics[width=0.27\textwidth]{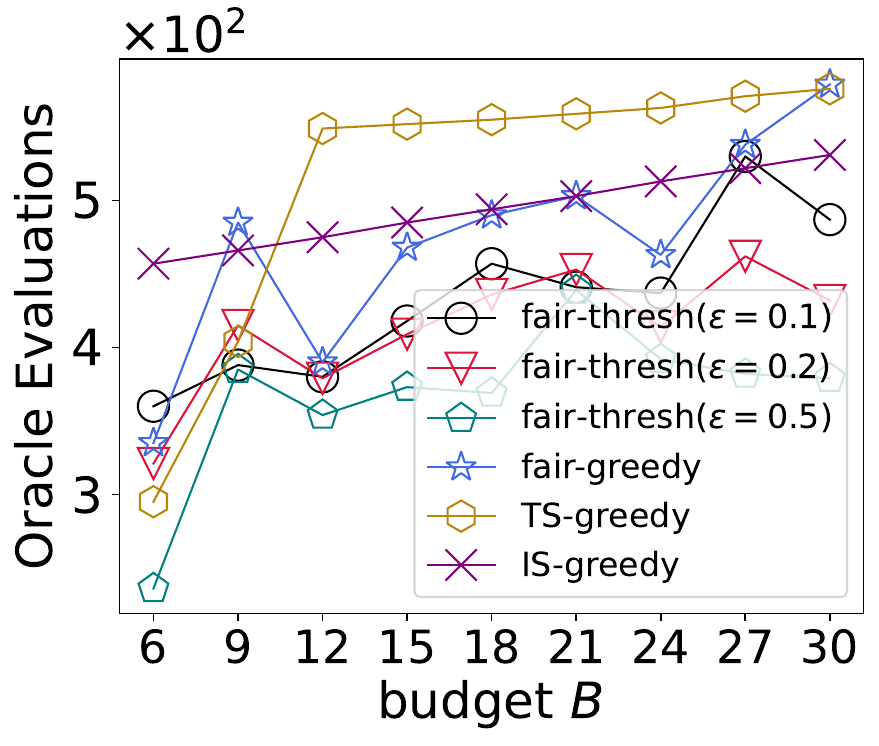}
\label{fig:sensor-2-delta2}
}
\caption{Experiments on sensor placement with $k$ types on Intel Lab sensor data: budget $B$ vs. Oracle Evaluations.}
\end{figure*}


\begin{figure*}[t]
\centering
\subfigure[Exact Oracles ($\delta=0$)]{
\centering
\includegraphics[width=0.30\textwidth]{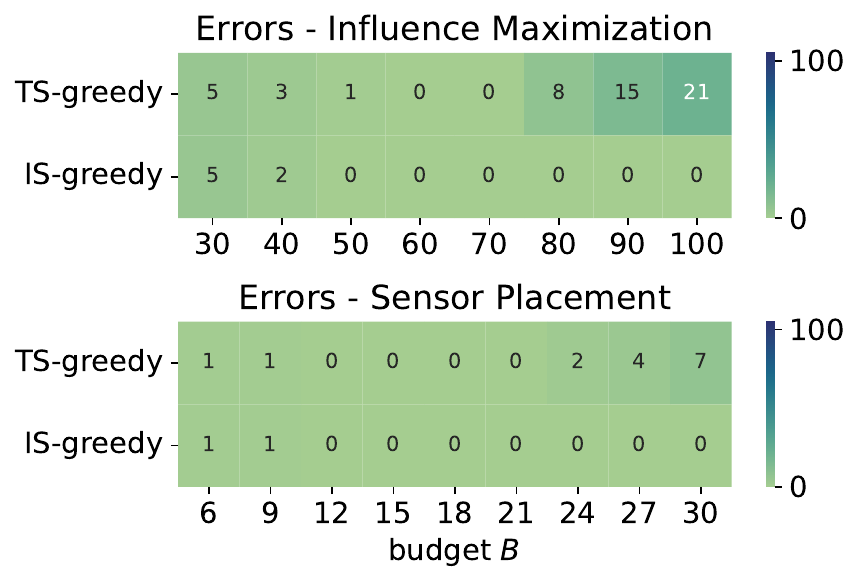}
\label{fig:errors}
}
~
\subfigure[Approximate Oracles ($\delta=0.1$)]{
\centering
\includegraphics[width=0.30\textwidth]{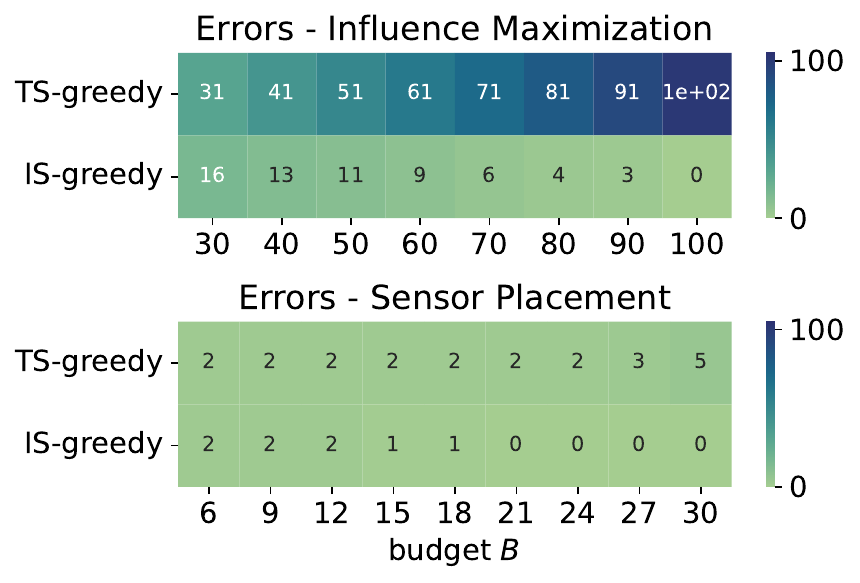}
\label{fig:errors-1}
}
~
\subfigure[Approximate Oracles ($\delta=0.2$)]{
\centering
\includegraphics[width=0.30\textwidth]{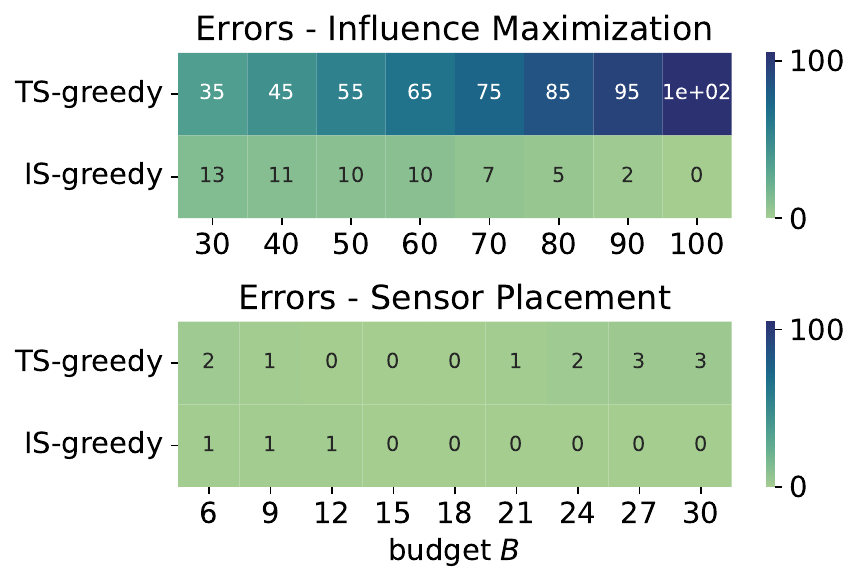}
\label{fig:errors-2}
}
\caption{Results of budget vs. errors for TS-greedy and IS-greedy algorithms.}
\end{figure*}


\subsection{Fair Sensor Placement with $k$ Types}
\label{sec:sensor}
We formulate Problem~\ref{problem:fair} in the context of the sensor placement with $k$ kinds of sensors.
Let $\Omega=$ $\left\{X_1, X_2, \ldots, X_n\right\}$ be a set of discrete random variables. The entropy of a subset $\mathbf{X}$ of $\Omega$ is defined as $H(\mathbf{X})=-\sum_{\mathbf{x} \in \mathrm{dom} X} \Pr[\mathbf{x}] \log \Pr[\mathbf{x}]$. The conditional entropy of $\Omega$ having observed $\mathbf{X}$ is $H(\Omega \mid \mathbf{X}):=H(\Omega)-H(\mathbf{X})$. 

In many sensor placement problems, the goal is to maximize the reduction of expected entropy.
Formally, the objective of this problem is to allocate a total of $B$ sensors of $k$ kinds to a set $V$ of $n$ locations, and each location can be instrumented with exactly one sensor. For each type $i \in [k]$, we place $B_i \in [\ell_i, u_i]$ sensors. Let $X_e^i$ be the random variable representing the observation collected from a sensor of the $i$-th kind if it is installed at the $e$-th location, and let $\Omega=\left\{X_e^i\right\}_{e \in V, i \in[k]}$. The problem of fair $k$-submodular maximization is to select a seed set $\mathbf{s} \in(k+1)^V$ such that $\argmax_{\mathbf{s} \in \mathcal{F}} H\left(\bigcup_{e \in \supp(\mathbf{s})}\left\{X_e^{\mathbf{s}(e)}\right\}\right). $


The entropy objective $f$ is monotone $k$-submodular \cite{ohsaka2015monotone}. When the entropy of a solution $\mathbf{s} \in (k+1)^V$ is not accessible but can be approximated within a $\delta$ error, i.e., the approximate influence spread $\Tilde{f}(\mathbf{s})$ satisfies $(1-\delta)f(\mathbf{s}) \leq \Tilde{f}(\mathbf{s}) \leq (1+\delta)f(\mathbf{s})$ for some $\delta \in [0,1)$. 

\subsubsection{\bf Experiment settings}
We use the preprocessed Intel Lab dataset~\cite{ohsaka2015monotone}, which contains approximately $2.3$ million readings collected from 54 sensors deployed in the Intel Berkeley research lab between February 28th and April 5th, 2004. Three kinds of measures, temperature, humidity, and light values, are extracted and discretized into bins of 2 degrees Celsius each, 5 points each, and 100 luxes each, respectively. 
Therefore, $k=3$ kinds of sensors are to be allocated to $n=54$ locations. For the fairness constraint, we set budget ranges for all three sensor measures to $[2, 10]$. We vary the total budget from $B=\{6, 9, 12, 15, 18, 21, 24, 27, 30\}$.

For the experiment of $\delta$-approximate $k$-submodular function cases, we simulate the approximate entropy of a subset of sensors with a random number in $[(1-\delta)f(\mathbf{s}), (1+\delta)f(\mathbf{s})]$ for a given error $\delta \in [0,1)$.

\subsubsection{\bf Results}  
The objective value comparison results of the influence maximization with $k$ topics are shown in Figs. \ref{fig:sensor-1}, \ref{fig:sensor-1-delta1}, and \ref{fig:sensor-1-delta2}. The oracle evaluation comparisons are presented in Figs. \ref{fig:sensor-2}, \ref{fig:sensor-2-delta1} and \ref{fig:sensor-2-delta2}.
The errors for each budget of the baseline algorithms are illustrated in the bottom parts of Figs. \ref{fig:errors},  \ref{fig:errors-1}, and \ref{fig:errors-2}.

Fig. \ref{fig:sensor-1} shows that all algorithms except the random assignment produce high objective values. As for oracle evaluations illustrated in Fig. \ref{fig:sensor-2}, {\sc TS-greedy} and {\sc Fair-Greedy} need the most oracle evaluations, and {\sc TS-greedy} will violate the fairness constraint when $B=6,9,24,27$ and $30$. However, our faster version {\sc Fair-Threshold} with $\epsilon=0.1,0.2$ and $0.5$ significantly reduces the number of oracle evaluations from {\sc Fair-Greedy}. Furthermore, when the total budget is large, our algorithm {\sc Fair-Threshold} needs significantly fewer oracle evaluations.

For the experiments with $\delta$-approximate oracles ($\delta \neq 0$), the objective values are comparable to exact oracle excess settings when $B$ is large. Notably, our algorithms produce similar quality solutions to the baseline algorithms. As observed in the application of influence maximization, the number of oracle evaluations does not strictly increase as the budget increases. Still, the faster version of our algorithm {\sc Fair-Threshold} needs the least number of oracle calls. As for the errors of baselines, the {\sc TS-greedy} algorithm is more biased than {\sc IS-greedy}.

%% file: sections/5-conclusion.tex
\section{Conclusion}
\label{sec:conclusion}
This paper initializes the study of the fairness in $k$-submodular maximization and investigates the application scenarios in influence diffusion and sensor placement. We proposed efficient fair approximation algorithms that match the best-known algorithms without a fairness constraint. Our theoretical results for approximate oracles expand the scalability of our algorithms to extensive real-world applications. This research fosters fairness and diversity in decision-making and machine-learning modeling.